%% file: main.tex
\documentclass[conference]{IEEEtran}
\IEEEoverridecommandlockouts
\usepackage{amsmath,amssymb,amsfonts}
\usepackage{algorithmic}
\usepackage{graphicx}
\usepackage{textcomp}
\usepackage{xcolor}
\input{packages}

\input{macros}

\usepackage[final]{pdfpages}
\usepackage{lipsum} 

\def\BibTeX{{\rm B\kern-.05em{\sc i\kern-.025em b}\kern-.08em
    T\kern-.1667em\lower.7ex\hbox{E}\kern-.125emX}}
\begin{document}

\title{An Information-Theoretic Analysis of Bayesian Reinforcement Learning
\thanks{This work was supported in part by the Knut and Alice Wallenberg Foundation, the Swedish Foundation for Strategic Research, and the Swedish Research Council under contract 2019-03606. %The extended version of the paper, with all the proofs and supplementary material, can be found in \textcolor{red}{Arxiv link}
}}

\author{\IEEEauthorblockN{Amaury Gouverneur, Borja Rodríguez-Gálvez, Tobias J. Oechtering, and Mikael Skoglund} \IEEEauthorblockA{Division of Information Science and Engineering (ISE)}\IEEEauthorblockA{KTH Royal Institute of Technology}\IEEEauthorblockA{\texttt{\{amauryg,borjarg,oech,skoglund\}@kth.se}}}

\maketitle

\begin{abstract}
Building on the framework introduced by Xu and Raginksy~\cite{xu2020minimum} for supervised learning problems, we study the best achievable performance for model-based Bayesian reinforcement learning problems. With this purpose, we define minimum Bayesian regret (MBR) as the difference between the maximum expected cumulative reward obtainable either by learning from the collected data or by knowing the environment and its dynamics.
%and the maximum expected cumulative reward that could be reached if the environment and its dynamics were known. 
%We present a definition of the minimum Bayesian regret for 
We specialize this definition to reinforcement learning problems modeled as Markov decision processes (MDPs) whose kernel parameters are unknown to the agent %. The
and whose uncertainty %of those parameters 
is expressed by a prior distribution.  
One method for deriving upper bounds on the MBR is presented and specific bounds based on the relative entropy and the Wasserstein distance are given. We then focus on two particular cases of MDPs, the multi-armed bandit problem (MAB) and the online optimization with partial feedback problem. For the latter problem, we show that our bounds can recover from below the current information-theoretic bounds by Russo and Van Roy~\cite{russo2016information}. 
\end{abstract}


\section{Introduction}
\label{sec:introduction}
In model-based reinforcement learning problems~\cite{sutton1999policy,bertsekas1996neuro}, an agent interacts sequentially with a dynamic environment by taking actions in order  to maximize its long-term performance.% The agent learns from the collected samples generated by its sequential interaction with the system.

This paper, as most related work in this field, focuses on systems and control objectives that are modeled as finite time horizon \emph{Markov decision processes} (MDPs). At each time $t = 1, \ldots, T$, the agent observes the environment state $S_t$ and takes an action $A_t$ following a decision policy $\varphi_t$. Independently of the action, the environment produces a random outcome $Y_t$. The reward is obtained as a deterministic function of the system’s outcome and the chosen action, $R_t = r(Y_t,A_t)$. The data is collected in a history $H^{t+1} = (S_1,A_1,R_1, \ldots , S_{t},A_{t},R_{t})$ and the system evolves to a state $S_{t+1}$. The procedure then repeats until the end of the time horizon, $t=T$. 

In the Bayesian setting, the MDP model $\Phi$ is treated as a random element of some parametric model family, which is drawn according to a prior distribution of the environment parameters $\Theta$. 
The goal of the agent is to identify a policy that yields the highest expected cumulative reward $\bE[\sum_{t=1}^Tr(Y_t,A_t)]$ under the uncertainty of these parameters.

\textcolor{black}{
The decision-making process in Bayesian reinforcement learning is typically more computationally demanding than the frequentist approach, however this setting presents various advantages as it facilitates regularization, handles parameter uncertainty naturally, and provides ways to solve the exploration-exploitation trade-off~\cite{ghavamzadeh2015bayesian}.
}

Following the work from Xu and Raginsky on Bayesian supervised learning~\cite{xu2020minimum}, \textcolor{black}{ we put aside the computational aspect to study the best achievable performance for model-based Bayesian reinforcement learning.} We define the \textit{minimum Bayesian regret} as the difference between the \textit{Bayesian cumulative reward} $R_\phi(\kappa_H)$, defined as the maximum expected cumulative reward  attainable by learning from the sequentially collected data, and $R_\phi(\kappa_\Theta)$, the maximum expected cumulative reward that could be reached if the environment parameters were known. 
We develop information-theoretic upper bounds on the minimum Bayesian regret under various assumptions for the reward function using the relative entropy and the Wasserstein distance. 

%\subsection{Structure of the paper}

\paragraph*{Structure of the paper} Section~\ref{sec:contributions} summarizes the contributions of this paper. The notations are introduced in Section~\ref{sec:notations}. Section~\ref{sec:model_definitions} presents the different models of decision processes studied and gives the definition of the Bayesian cumulative reward and the minimum Bayesian regret. Section~\ref{sec:upper_bounds} and~\ref{sec:upper_bounds_MAB} are devoted to information-theoretic upper bounds on the MBR. Finally, conclusions are presented in Section~\ref{sec:conclusion}.

\section{Contributions}
\label{sec:contributions}

%\textcolor{black}{Consider giving more weight on the theoretical framework developed as a contribution.}

%In this work, we expended the framework introduced by Xu and Raginsky~\cite{xu2020minimum} to the study of the best achievable performance for model-based reinforcement learning problems. In this regard, we: 
In this work, inspired by Xu and Raginsky's~\cite{xu2020minimum} framework on the study of the best achievable performance of supervised learning problems, we propose an analogous framework for the study of model-based reinforcement learning problems. %In this regard, we:
\textcolor{black}{Our contributions in this regard can be summarized as:}
\begin{color}{black}
\begin{enumerate}
    \item \textcolor{black}{Developing a theoretical framework of model-based Bayesian MDPs suited for information-theoretic studies.}
    \item Proposing a definition of the minimum Bayesian regret (MBR) for reinforcement learning problems modeled as Markov decision processes.
    \item Presenting a data processing inequality for the Bayesian cumulative reward in Lemma~\ref{lem:data_process}.
    \item Formulating upper bounds on the $\textnormal{MBR}$ for general MDPs based on the relative entropy (Proposition~\ref{prop:kl_div_subgaussian}) and the Wasserstein distance (Proposition~\ref{prop:wasserstein}). We present particular cases of these bounds for the case of bounded reward functions in Corollaries~\ref{cor:kl_div_bounded_reward} and~\ref{cor:wasserstein}, and the tightness of these results are compared in Remark~\ref{rem:wasserstein_is_tighter}.
    \item Deriving MBR bounds for the multi-armed bandit and for the online optimization with partial feedback problems. In this last setting, we show how our bound recovers \emph{from below} results from Russo and Van Roy~\cite{russo2016information}. 
\end{enumerate}
\end{color}

\section{Notations and preliminaries}
\label{sec:notations}

Throughout the paper, random variables $X$ are written in capital letters, their realizations $x$ in lower-case letters, and their set of outcomes $\cX$ in calligraphic letters. The probability distributions of a random variable $X$ is denoted as $\bP_X$. When more than one random variable is considered, e.g., $X$ and $Y$, we use $\bP_{X,Y}$ to denote their joint distribution and $\bP_X \bP_Y$ for their product distribution\footnote{Note that this slight abuse of notation does not mean that the product distribution is the product of the distributions.}. We write the conditional probability distribution of $Y$ given $X$ as $\bP_{Y|X}$, defining a probability distribution $\bP_{Y|X=x}$ over $\cY$ for each element $x\in \cX$. 

We use the underscore notation $X_t$ to represent a random variable at time $t=1,\ldots, T$ and  the exponent notation $X^t$ to denote a sequence of random variables $X^t \equiv (X_1,\ldots,X_t)$ for $t=2,\ldots,T$. For consistency we let $X^1 \equiv X_1$.

The relative entropy between two probability distributions $\bP$ and $\bQ$ is defined as $\KL{\bP}{\bQ} := \int \log \big( \frac{d\bP}{d\bQ} \big) d\bP$ if $\bP$ is absolutely continuous with respect to $\bQ$ and $\KL{\bP}{\bQ} \to \infty$ otherwise. The notation $d\bP/d\bQ$ is the Radon-Nikodym derivative. Similarly, the mutual information between $X$ and $Y$ is defined as $\mi(X;Y) := \KL{\bP_{X,Y}}{\bP_{X} \bP_{Y}}$, and the conditional mutual information between $X$ and $Y$, given $Z$, as $\mi(X;Y|Z) := \bE[\mi(X;Y|Z=z)]$, where $\mi(X;Y|Z=z) := \KL{\bP_{X,Y|Z=z}}{\bP_{X|Z=z}  \bP_{Y|Z=z}}$.

Finally, if two probability distributions $\bP$ and $\bQ$ are defined in a Polish space $\cX$ with respect to a metric $\rho$, then their Wasserstein distance of order $p \geq 1$ is $\bW_p(\bP,\bQ) := \big(\inf_{\bD \in \Pi(\bP,\bQ)} \int \rho d\bD \big)^{1/p}$, where $\Pi(\bP,\bQ)$ is the set of all couplings of $\bP$ and $\bQ$; i.e., all joint distributions on $\cX \times \cX$ with marginals $\bP$ and $\bQ$. As this work is focused on upper bounds and since by H\"older's inequality $\bW_p \leq \bW_q$ for all $p \leq q$~\cite[Remark 6.6]{villani2009optimal}, in what follows, we will only be using the Wasserstein distance of order $1$, $\bW := \bW_1$. \textcolor{black}{For a discrete random variable $X$, the Shannon entropy is defined as $\ent(X) \coloneqq \bE[-\log(\bP_X(X))]$.}

\section{Model and Definitions}
\label{sec:model_definitions}

In this section we first introduce formally Markov decision processes. We then present the multi-armed bandit and the online optimization with partial feedback problems two special cases of MDPs. After that, we describe Bayesian cumulative reward and prove that it respects a data-processing inequality.
%show its data processing inequality property is given. 
Finally, we define minimum Bayesian regret. 

%This section first introduces formally the Markov decision process and presents the multi-armed bandit problem and the online optimization problem with partial feedback as two special cases of MDP. Then the Bayesian cumulative reward is described and its data inequality property is given. Finally, the minimum Bayesian regret is defined. 

\subsection{Markov Decision Process}
\label{subsec:markov_decision_process}
In a \emph{Markov decision process} (MDP), at each time step $1, \ldots, T$, an agent interacts with the environment by observing the system's state $S_t \in \cS$ and selecting accordingly an action $A_t \in \cA$. The system then produces an outcome $Y_t \in \cY$ which the agent associates with a scalar reward $R_t \in \bR$. 

In Bayesian reinforcement learning, the environment is completely characterized by a random variable $\Theta \in \cO$ with probability distribution $\bP_\Theta$. Therefore,  an MDP $\Phi$ is defined by a transition kernel $\kappa_{\textnormal{trans}} : \ccS \times (\cS \times \cA \times \cO) \to [0,1]$ such that $\bP_{S_{t+1}|S_t, A_t, \Theta} = \kappa_{\textnormal{trans}}(\cdot,(S_t, A_t, \Theta))$, an outcome kernel $\kappa_{\textnormal{out}} : \ccY \times (\cS \times \cO) \to [0,1]$ such that $\bP_{Y_{t}|S_t, \Theta} = \kappa_{\textnormal{out}}(\cdot,(S_t, \Theta))$, an initial state prior distribution $\bP_{S|\Theta}$ such that $S_1 \sim \bP_{S|\Theta}$, and a reward function $r: \cY \times \cA \to \bR$. The reward is a deterministic function of the system's outcome and the chosen action, hence there is a reward's kernel $\kappa_{\textnormal{reward}} : \cB(\bR) \times (\cS, \cA, \cO) \to [0,1]$ such that $\bP_{R_t|S_t, A_t, \Theta} = \kappa_{\textnormal{reward}}(\cdot,(S_t, A_t, \Theta))$.

The task in Bayesian reinforcement learning is to learn a policy $\varphi = \lbrace \varphi_t: \cS \times \cH^t \to \cA \rbrace_{t=1}^T$ taking an action $A_t$ based on the current observation $S_t$ and the past collected data $H^t$, where $H_{t+1} = (S_{t}, A_{t}, R_{t})$ that maximizes the \emph{cumulative expected reward} $r_\Phi(\varphi) \coloneqq \bE \big[ \sum_{t=1}^T r\big(Y_t, \varphi_t(S_t, H^t) \big) \big]$.

\subsection{Static state MDP and multi-armed bandit problem}

%\textcolor{black}{static MDP (w.l.o.g. with state $S=0$  ) whose outcome are independent of the  state. $\implies$ no need to introduce new definitions}

A \emph{static state} MDP is an MDP whose transition kernel $\kappa_\textnormal{trans}$ is such that the system's state remains constant, i.e. $ S_t = S$ for all $t$. We will use the notation $\Pi$ to refer to such MDP. 

A \emph{multi-armed bandit} (MAB) problem can be formalized as a static state MDP whose environment parameters $\Theta$ and outcomes $Y_t$ are independent of $S$. Similarly to an MDP, the task in a MAB problem is to learn a policy $\varphi = \lbrace \varphi_t: \cS \times \cH^t \to \cA \rbrace_{t=1}^T$ taking an action $A_t$ based on the past collected data $H^t$, where $H_{t+1} = (A_{t}, R_{t})$ that maximizes the cumulative expected reward $r_\Pi(\varphi) \coloneqq \bE \big[ \sum_{t=1}^T r\big(Y_t, \varphi_t(S,H^t) \big) \big]$.

%It is fully defined by an outcome kernel $\kappa_{\textnormal{out}} : \ccY \times \cO \to [0,1]$ such that $\bP_{Y_{t}|\Theta} = \kappa_{\textnormal{out}}(\cdot,\Theta)$, and a deterministic reward function $r: \cY \times \cA \to \bR$. From the outcome kernel and the reward function, one can construct a reward's kernel $\kappa_{\textnormal{reward}} : \cB(\bR) \times (\cA, \cO) \to [0,1]$ such that $\bP_{R_t|A_t, \Theta} = \kappa_{\textnormal{reward}}(\cdot,(A_t, \Theta))$.\\

A variant of that problem is the \emph{online optimization problem with partial feedback} studied by Russo and Van Roy~\cite{russo2016information}. This problem can also be modeled as a static MDP $\Pi$ with a finite action space $\cA$ where at each time $t=1,\ldots,T$, the agent selects an action $A_t$ and observes a ``per-action outcome" $Y_{t,A_t} \in \cY'$ giving rise to past collected data $H^t$ with $H_{t+1} = (A_{t}, Y_{t,A_t})$.
The agent associates the ``per-action outcome" with a reward $R_t = r'(Y_{t,A_t})$ through a preference function $r': \cY' \to \bR$. In this setting,  the random outcome $Y_t \in \cY$  is the vector formed with all the possible outcomes, $Y_t \equiv \lbrace Y_{t,a}\rbrace_{a\in \cA}$ and the reward function $r:\cY \times \cA \to \bR$ is a function such that for all $Y_t \in \cY$ and $A_t\in \cA$, we have $r(Y_t,A_t) = r'(Y_{t,A_t})$. In this problem, as well, the environment parameters $\Theta$ and outcomes $Y_t$ are independent of $S$.

%random outcome $Y_{t,a}$ is associated with each action $a \in \cA$ and time step $t=1,\ldots,T$. The system's outcome is a vector formed with all the possible outcomes, $Y_t \equiv \lbrace Y_{t,a}\rbrace_{a\in \cA}$. At each time $t=1,\ldots,T$, the agent associates a reward with each outcome  and observes the system outcome, $Y_{t,A_t}$, giving rise to past collected data $H^t$ with $H_{t+1} = (A_{t}, Y_{t,A_t})$.

\subsection{The Bayesian Cumulative Reward}
\label{subsec:BCR}

A decision policy that maximizes the expected cumulative reward among all policies is called a \emph{Bayesian decision policy}. The corresponding maximum expected cumulative reward is defined as the \emph{Bayesian cumulative reward}.

\begin{definition}
\label{def:bcr}
The \emph{Bayesian cumulative reward} (BCR) of a Markov decision process $\Phi$ is defined as $R_\Phi \coloneqq \sup_{\varphi} r_\Phi(\varphi)$, where the supremum is taken over the collection $\varphi$ of all decision rules $\varphi_t: \cS \times \cH^t \to \cA$ such that the expectation is defined.
\end{definition}

The notion of Bayesian cumulative reward can be generalized to allow the agent to select an action using some knowledge $X^t$ such that each $X_{t+1}$ is obtained from $(S_{t}, A_{t}, Y_{t}, \Theta)$. In this generalized model, the knowledge $X_t$ is obtained through a knowledge kernel $\kappa_\textnormal{know}: \ccX \times (\cS \times \cA \times \cY \times \cO) \to [0,1]$ such that $\bP_{X_{t+1} | S_{t}, A_{t}, Y_{t}, \Theta } = \kappa_{\textnormal{know}}(\cdot, (S_{t}, A_{t}, Y_{t}, \Theta))$. Now, let $\varphi = \lbrace \varphi_t: \cS \times \cX^t \to \cA \rbrace_{t=1}^T$ be a policy in this relaxed setting. 
Then, the \emph{generalized Bayesian cumulative reward} (also written as BCR when no confusion is possible) of an MDP $\Phi$ with knowledge kernel $\kappa_\textnormal{know}$ is $R_\Phi(\kappa_\textnormal{know}) \coloneqq \sup_\varphi r_\Phi(\kappa_\textnormal{know},\varphi)$, where 
\begin{align*}
     r_\Phi(\kappa_\textnormal{know},\varphi) \coloneqq  \bE \bigg[ \sum_{t=1}^T r \big( Y_t, \varphi_t(S_t, X^t) \big) \bigg] 
\end{align*}
and again the supremum is taken over the collection $\varphi$ of all decision rules $\varphi_t :  \cS \times \cX^t \to \cA$ such that the expectation above is defined.

\begin{remark}
Given an MDP $\Phi$, let $\cX = \cS \times \cA \times \bR$ and  $\kappa_\textnormal{know}$ be a kernel such that $X_{t+1} = (S_t, A_t, R_t)$ and denote this kernel $\kappa_\textnormal{H}$. Note that $X_t = H_t$ and $R_\Phi (\kappa_\textnormal{H} ) = R_\Phi$.% In the case of static MDP $\Pi$, let $\kappa_\textnormal{know}$ be a kernel such that $X_{t+1} = (A_t, R_t)$ and denote this kernel $\kappa_\textnormal{h}$. Note that $R_\Pi (\kappa_\textnormal{h})$ corresponds to the Bayesian cumulative reward of multi-armed bandit problem.
\end{remark}

After defining the generalized Bayesian cumulative reward, one can study the case where the agent has access to some processed information $Z_t$ obtained from the knowledge $X^t$. Let $\kappa_\textnormal{process}$ denote a collection of processing kernels $\lbrace \kappa_\textnormal{process,t}: \ccZ \times (\cX^t) \to [0,1] \rbrace_{t=1}^T$ such that $\bP_{Z_{t} |X^{t}} = \kappa_\textnormal{process,t}\big(\cdot, (X^{t})\big)$ for each $t=1,\ldots,T$. 
Then the \emph{processed Bayesian cumulative reward} with knowledge kernel $\kappa_\textnormal{know}$ and process kernels $\kappa_\textnormal{process}$ is $R_\Phi(\kappa_\textnormal{know},\kappa_\textnormal{process} ) \coloneqq \sup_{\psi} r_\Phi(\kappa_\textnormal{know},\kappa_\textnormal{process},\psi)$, where 
\begin{align*}
    r_\Phi(\kappa_\textnormal{know},\kappa_\textnormal{process},\psi) \coloneqq \bE \bigg[ \sum_{t=1}^T r \big( Y_t, \psi_t(S_t, Z_t) \big) \bigg]
\end{align*}
and the supremum is taken over the collection $\psi$ of all decision rules $\psi = \lbrace \psi_t: \cS \times \cZ \to \cA \rbrace_{t=1}^T$ such that the expectation above is defined.

\subsection{Data processing inequality for the BCR}
\label{subsec:dataprocess}

An important property of the Bayesian cumulative reward is the data processing inequality (DPI), stating that no amount of processing of the knowledge random variables can increase the cumulative reward. This is formalized in the following lemma. 
\begin{restatable}{lemma}{dataprocess}
\label{lem:data_process}

Let $\kappa_\textnormal{U}$ be a knowledge kernel associated with an MDP $\Phi$ and $\kappa_{\textnormal{V}|\textnormal{U}}$ a collection of processing kernels. Then, the cumulative Bayesian reward using the knowledge from $U$ is at least as large as the processed Bayesian cumulative reward using the processed knowledge from $V$. More precisely,  
\begin{equation*}
    R_\Phi (\kappa_\textnormal{U} ) \geq R_\Phi (\kappa_\textnormal{U},\kappa_{\textnormal{V}|\textnormal{U}} )
\end{equation*}
\end{restatable}

\begin{proof}[Intuition of the proof]
The proof follows by iteratively employing~\cite[Lemma~3.22]{kallenberg2005probabilistic} in a similar fashion to~\cite[Lemma~1]{xu2020minimum} and taking care that the random objects in the definitions of $R_\Phi(\kappa_U)$ and $R_\Phi(\kappa_U,\kappa_{V|U})$ follow the distributions described by the dynamics of the MDP $\Phi$ and their respective actions $\varphi_t$ and $\psi_t$. The complete proof is in appendix~\ref{sec:proofs_lemas}.
\end{proof}

\subsection{The Minimum Bayesian Regret (MBR)}
\label{subsec:minimum_bayesian_regret}
%
%Using the previously introduced generalized Bayesian cumulative reward 
We define the \emph{fundamental limit of the Bayesian cumulative reward} as the Bayesian cumulative reward for a knowledge kernel such that $X_t = \Theta$, that is when the environment parameters are known to the agent. We denote such a kernel as $\kappa_{\Theta}$.  

%\textcolor{black}{\textbf{For journal version:} Prove that this is indeed the supremum.}

\begin{definition}
The \emph{fundamental limit of the Bayesian cumulative reward} of a Markov decision process $\Phi$ is defined as
\begin{align*}
    R_\Phi(\kappa_\Theta) \coloneqq \sup_{ \lbrace \psi_t \rbrace_{t=1}^T} \bE \bigg[ \sum_{t=1}^T r \big( Y_t, \psi_t(S_t, \Theta) \big) \bigg],
\end{align*}
where the kernel $\kappa_\Theta$ is such that $X_t = \Theta$ for all $t=1,\ldots,T$.
\label{def:fl_bcr}
\end{definition}
\begin{assumption}
For the rest of the paper, we will assume that the supremum from~\Cref{def:fl_bcr} exists and we will denote by  $\psi^\star = \lbrace \psi^\star_t \rbrace_{t=1}^T$ a policy that achieves it.
\end{assumption}

We define the gap between this limit and the Bayesian cumulative reward as the \emph{minimum Bayesian regret}. 
\begin{definition}
\label{def:mbr}
The \emph{minimum Bayesian regret (MBR)} of a Markov decision process $\Phi$ is defined as 
\begin{equation*}
    \textnormal{MBR}_\Phi \coloneqq R_\Phi (\kappa_{\Theta} ) - R_\Phi (\kappa_{\textnormal{H}} ).
\end{equation*}

\end{definition}
\textcolor{black}{
The MBR characterizes the regret of the optimal decision policy that has access to the collected data, but not to environment parameters, and is therefore an algorithm-independent quantity. It can be interpreted as the inherent difficulty of the reinforcement learning problem resulting from
the lack of knowledge about the environment parameters $\Theta$.
}
\section{Upper bounds on the Minimum Bayesian Regret}
\label{sec:upper_bounds}

%In this section, we present several information-theoretic upper bounds on the minimum Bayesian regret.
%First, we present results in the case of sub-Gaussian reward functions in terms of the relative entropy. Then, we turn to reward functions that are $L$-Lipschitz under a metric $\rho$ and the presented bounds are in terms of the Wasserstein distance. Finally, we show how those results can be connected for bounded reward functions.

%Throughout this section, we will use $\psi^\star = \lbrace \psi^\star_t \rbrace_{t=1}^T$ to denote a policy that achieves the supremum of $R_\Phi(\kappa_\Theta)$. 
%We let $Y^\star_t$ and $S^\star_t$ be the outcomes and state obtained from the actions from $\psi^\star_t$, the kernels that describe the MDP $\Phi$ with knowledge kernel $\kappa_\Theta$ and let $\hat{Y}_t$ and $\hat{S}_t$ be the outcomes and state obtained from the actions from $\psi^\star_t$, the kernels that describe the MDP $\Phi$ with knowledge kernel $\kappa_\textnormal{H}$ and processing kernels $\kappa_{\Theta|\textnormal{H}}$.  

In this section, we start by giving an upper bound of the minimum Bayesian regret in terms of the difference of the fundamental limit of the BCR, $R_\Phi(\kappa_\Theta)$, and the processed BCR with the optimal Bayes parameters' estimator $\bP_{\Theta|H^t}$ as the processing kernel and the optimal policy of $R_\Phi(\kappa_\Theta)$. That is, the difference between the best obtainable risk knowing the environment parameters $\Theta$, and the best obtainable risk inferring the parameters with an optimal estimator. This bound, in turn, can be developed into a bound that compares the sum of the individual terms in the optimal trajectory of $R_\Phi(\kappa_\Theta)$ and those obtained with the processing kernels $\bP_{\Theta|H^t}$. This way, we can employ similar techniques to those in the literature (e.g., \cite{xu2017information, rodriguez2021tighter}) and bound the MBR in terms of the sum of terms depending on the statistical difference between the distributions of those two trajectories.

\subsection{The Thompson sampling regret}
%The next lemma serves as a first step to upper bound the minimum Bayesian regret.  It bounds the minimum Bayesian regret by the difference between $R_\Phi(\kappa_\Theta)$ and $r_\Phi(\kappa_\textnormal{H},\kappa_{\Theta|\textnormal{H}},\psi^\star)$, the expected cumulative reward obtained using the same policy with estimates $\hat{\Theta}_t$ of $\Theta$ given the collected history.  In the case of a static Markov decision process $\Pi$,  $r_\Pi(\kappa_\textnormal{H},\kappa_{\Theta|\textnormal{H}},\psi^\star)$ is equivalent to the expected cumulative reward of the Thompson sampling algorithm. 

Consider the fundamental limit of BCR, $R_\Phi(\kappa_\Theta)$, and its optimal trajectory $\psi^\star$. A natural algorithm to try to solve an MDP $\Phi$ when environment $\Theta$ is unknown is to estimate the environment parameters with some processing kernel of the history $\kappa_{\Theta|\textnormal{H}}$ and select an optimal action based on such processing. An elegant scenario would be to have the additional information of knowing which is the optimal trajectory $\psi^\star$ and to be able to calculate the Bayes optimal estimator $\bP_{\Theta|H^t}$ to process the history. In fact, for a static MDP $\Pi$, this algorithm is studied in the literature and is known as the Thompson's sampling algorithm~%\textcolor{black}{TO DO: add a couple of references}
\textcolor{black}{~\cite{thompson1933likelihood,scott2010modern,chapelle2011empirical,may2012optimistic,osband2013more,russo2016information}}.
Therefore, the next lemma shows that the MBR is bounded from above by the difference of $R_\Phi(\kappa_\Theta)$ and the BCR of such an algorithm, $r_\Phi(\kappa_\textnormal{H}, \kappa_{\Theta|\textnormal{H}, \psi^\star})$.

\begin{lemma}
\label{lemma:thompson_sampling} For any MDP $\Phi$, the MBR can be upper bounded as follows, 
\begin{equation*}
\textnormal{MBR}_\Phi \leq R_\Phi(\kappa_\Theta) - r_\Phi(\kappa_\textnormal{H},\kappa_{\Theta|\textnormal{H}},\psi^\star).
    %\textnormal{MBR}_\Phi \leq \bE \Big[\sum_{t=1}^T r(Y_t,\psi_t^\star(S_t,\Theta)) \Big]-\bE \Big[\sum_{t=1}^T r(Y_t,\psi_t^\star(S_t,\hat{\Theta}_t)) \Big]
\end{equation*}
\end{lemma}

\begin{proof}
The proof starts by using Lemma~\ref{lem:data_process} to lower bound $R_\Phi (\kappa_{\textnormal{H}} )$ with $R_\Phi (\kappa_{\textnormal{H}},\kappa_{\Theta|\textnormal{H}} ) $. The last inequality follows from the definition of $R_\Phi (\kappa_{\textnormal{H}},\kappa_{\Theta|\textnormal{H}} ) $ being the supremum over $\psi$ of $r_\Phi(\kappa_\textnormal{H},\kappa_{\Theta|\textnormal{H}},\psi)$. More precisely,
\begin{align*}
    \textnormal{MBR}_\Phi &= R_\Phi (\kappa_{\Theta} ) - R_\Phi (\kappa_{\textnormal{H}} )\\
     &\stackrel{}{\leq} R_\Phi (\kappa_{\Theta} )- R_\Phi (\kappa_{\textnormal{H}},\kappa_{\Theta|\textnormal{H}} )\\
     &\stackrel{}{\leq} R_\Phi(\kappa_\Theta) - r_\Phi(\kappa_\textnormal{H},\kappa_{\Theta|\textnormal{H}},\psi^\star).
     %\bE \Big[\sum_{t=1}^T r(Y^\star_t,\varphi_t^\star(S^\star_t,\Theta))\Big]-\bE\Big[\sum_{t=1}^T r(Y_t,\varphi_t^\star(S_t,\hat{\Theta}_t))\Big]
\end{align*}

\end{proof}

In what follows, we will use the notations $Y^\star_t$ and $S^\star_t$ for the outcomes and states obtained from the actions derived from $\psi^\star$, the kernels that describe the MDP $\Phi$, and the knowledge kernel $\kappa_\Theta$.
Similarly we will let $\hat{Y}_t$, $\hat{S}_t$ and $\hat{H}_t$ be the outcomes, states, and histories obtained from the actions derived from $\psi^\star$, the kernels that describe the MDP $\Phi$ with knowledge kernel $\kappa_\textnormal{H}$, and processing kernels $\kappa_{\Theta|\textnormal{H}}$.  
The following lemma builds on Lemma~\ref{lemma:thompson_sampling} and shows how the MBR can be written as the sum of the individual differences of the expected rewards obtained following the optimal trajectory $(Y^\star_t, S^\star_t)_{t=1}^T$ and the trajectory of the aforementioned algorithm $(\hat{Y}_t, \hat{S}_t)_{t=1}^T$ given the history $\hat{H}^t$.

\begin{comment}
\begin{lemma}
\label{lemma:mdp_diff_expectations_thompson}
For any MDP $\Phi$, the MBR can be upper bounded as follows, 
\begin{align*}
\textnormal{M}&\textnormal{BR}_\Phi \leq \\
    &\sum_{t=1}^T \bE \bigg[ \bE \Big[ r\big(Y^\star_t, \psi^\star_t (S^\star_t, \Theta)\big) - r\big(\hat{Y}_t, \psi^\star_t (\hat{S}_t, \hat{\Theta}) \big)\Big| \Theta, \hat{\Theta}, H^t \Big] \bigg].
\end{align*}
\end{lemma}
\end{comment}

Unrolling $R_\Phi(\kappa_\Theta)$ and $r_\Phi(\kappa_\textnormal{H}, \kappa_{\textnormal{H}|\Theta}, \psi^\star)$ and using the linearity of the expectation and the law of total expectation reveals that the right-hand side term from Lemma~\ref{lemma:thompson_sampling} can be written as
\begin{equation}
    \sum_{t=1}^T \bE \bigg[ \bE \Big[ r\big(Y^\star_t, \psi^\star_t (S^\star_t, \Theta)\big) - r\big(\hat{Y}_t, \psi^\star_t (\hat{S}_t, \hat{\Theta}_t) \big)\Big| \Theta, \hat{\Theta}_t, \hat{H}^t \Big] \bigg].
    \label{eq:mdp_diff_expectations_thompson}
\end{equation}
%The lemma follows simply by unrolling $R_\Phi(\kappa_\Theta)$ and $r_\Phi(\kappa_\textnormal{H}, \kappa_{\textnormal{H}|\Theta}, \psi^\star)$ and using the linearity of the expectation and the law of total expectation. The importance of the above lemma 
\begin{remark}
\label{rem:diff_dists}
The importance of this re-formulation lays in the fact that the first term inside the conditional expectation is distributed according to $\bP_{Y^\star,S^\star|\Theta}$ since $(Y^\star,S^\star)$ are independent of the history $\hat{H}^t$ when the environment parameters $\Theta$ are known. Similarly, the second term is distributed according to $\bP_{\hat{Y}_t,\hat{S}_t|\hat{H}^t}$ since $(\hat{Y}_t,\hat{S}_t)$ are independent of the sampled parameters $\hat{\Theta}_t$ when the history is known. Both facts follow from the Markov chain $(Y^\star_t,S^\star_t) - \Theta - (\hat{Y}_t, \hat{S}_t) - \hat{H}^t - \hat{\Theta}_t$. Therefore, conditioned on the history $\hat{H}^t$ and the environment parameters $\Theta, \hat{\Theta}_t$, the terms in the sum of~\eqref{eq:mdp_diff_expectations_thompson} are a difference of expectations of random objects which randomness comes from distributions on the same space, which permits us to employ known decoupling techniques to bound these differences in terms of such distributions.
\end{remark}

In the sequel, we use this fact to bound the MBR as the sum of terms depending on the statistical difference between the distributions of the elements from the optimal trajectory $(Y^\star_t, S^\star_t)|\Theta$ and the trajectory described by the algorithm with the Bayes optimal parameters' estimator $(\hat{Y}_t, \hat{S}_t)|\hat{H}^t$. More precisely, we use the techniques from e.g.~\cite{russo2016information,xu2017information} when the reward is sub-Gaussian, from e.g.~\cite{rodriguez2021tighter,wang2019information} when it is Lipschitz, and from~\cite{rodriguez2021tighter} to connect both settings when the reward is bounded.

\subsection{Sub-Gaussian reward functions}
\label{subsec:subgaussian}

We consider arbitrary reward functions $r:\cY \times \cA \to \bR$ mapping an outcome and an action to a scalar reward. %Under the assumption that the random reward $r(\hat{Y}_t,\psi_t^\star(\hat{S}_t,\theta))$ is $\sigma_t^2$-sub-Gaussian under $\bP_{\hat{Y}_t,\hat{S}_t|\hat{H}^t = \hat{h}^t}$ for all $\theta \in \cO$ and all $\hat{h}^t \in \cH^t$,  we can bound the $\textnormal{MBR}_{\Phi}$ in terms of the relative entropy between the optimal trajectory $Y^\star_t, S^\star_t$ given the environment parameters $\Theta$ and the Thompson sampled trajectory $\hat{Y}_t, \hat{S}_t$ given the collected history $\hat{H}^t$. This is formalized in Proposition \ref{prop:kl_div_subgaussian}. 
Under the assumption that the random reward $r(\hat{Y}_t,\psi_t^\star(\hat{S}_t,\theta))$ is $\sigma_t^2$-sub-Gaussian under $\bP_{\hat{Y}_t,\hat{S}_t|\hat{H}^t = \hat{h}^t}$ for all $\theta \in \cO$ and all $\hat{h}^t \in \cH^t$, the $\textnormal{MBR}_\Phi$ is bounded by a sum of terms related to the relative entropy between the distribution of the elements of each step of the optimal trajectory, i.e., $Y^\star_t, S^\star_t$, and the Thompson's sampled trajectory, i.e., $\hat{Y}_t, \hat{S}_t$. This is formalized in the following Proposition.

\begin{restatable}{proposition}{kldivboundsubgaussian}
\label{prop:kl_div_subgaussian}
If for all $t = 1,\ldots,T$, the random reward $r(\hat{Y}_t,\psi_t^\star(\hat{S}_t,\theta))$ is $\sigma_t^2$-sub-Gaussian under $\bP_{\hat{Y}_t,\hat{S}_t|\hat{H}^t = \hat{h}^t}$ for all $\theta \in \cO$ and all $\hat{h}^t \in \cH^t$, then, 
\begin{align*}
    \textnormal{MBR}_{\Phi} \leq \sum_{t=1}^T  \bE \Big[ \sqrt{2 \sigma_t^2 \KL{\bP_{Y^\star_t, S^\star_t|\Theta}}{\bP_{\hat{Y}_t,\hat{S}_t|\hat{H}^t}}}  \Big].
\end{align*}
\end{restatable}

\begin{proof}
\textcolor{black}{The proof follows from applying Donsker-Varadhan's inequality~\cite[Theorem~5.2.1]{gray2011entropy} to~\eqref{eq:mdp_diff_expectations_thompson} using~\Cref{rem:diff_dists} in a similar fashion to~\cite{russo2016information,xu2017information}.}
\end{proof}

\subsection{Lipschitz reward functions}
\label{subsec:lipschitz}
In this subsection, we suppose that the set of outcomes and actions $(\cY,\cA)$ together with the metric $\rho:(\cY\times \cA) \times (\cY\times \cA)  \rightarrow \bR_+$, form a Polish metric space.% $((\cY \times \cA),\rho)$. 

Assume that the reward function $r: \cY \times \cA \to \bR$ is $L$-Lipschitz under the metric $\rho$,  that is that $|r(y,a) - r(y',a')| \leq L \rho((y,a), (y',a'))$ for all $y,y' \in \cY$ and $a,a' \in \cA$. Under this assumption, the Wasserstein distance can be used to upper bound the minimum Bayesian regret.

\begin{restatable}{proposition}{wasserstein}
\label{prop:wasserstein}
Suppose that $(\cY\times\cA)$ is a metric space with metric $\rho$.  If the reward function $r: \cY \times \cA \to \bR$ is $L$-Lipschitz under the metric $\rho$, then
\begin{equation*}
      \textnormal{MBR}_{\Phi}  \leq L \sum_{t=1}^T  \bE \big[\bW(\bP_{Y^\star_t, S^\star_t|\Theta}, \bP_{\hat{Y}_t,\hat{S}_t|\hat{H}^t}) \big].
\end{equation*}
%where the expectation is taken with respect to  $\bP_{\Theta,\hat{H}^t}$.
\end{restatable}
\begin{proof}
\textcolor{black}{The proof follows from applying  Kantorovich–Rubinstein duality~\cite[Remark~6.5]{villani2009optimal} to~\eqref{eq:mdp_diff_expectations_thompson} using~\Cref{rem:diff_dists} analogously  to~\cite{rodriguez2021tighter,wang2019information}.}
\end{proof}

\subsection{Bounded reward functions}
\label{subsec:bounded}

We can obtain upper bounds on the minimum Bayesian regret for bounded reward functions as particular cases of both Proposition~\ref{prop:kl_div_subgaussian} and Proposition~\ref{prop:wasserstein}. We will consider without loss of generality reward functions bounded in $[0,1]$.

First,  from Hoeffding's lemma~\cite[Theorem~1]{hoeffding1994probability},  we have that if $r:\cY \times \cA \to [0,1]$ then the reward is $1/4$-sub-Gaussian under any distribution of the arguments.  This fact and Proposition~\ref{prop:kl_div_subgaussian} leads to Corollary~\ref{cor:kl_div_bounded_reward}.

\begin{restatable}{corollary}{kldivboundboundedcor}
\label{cor:kl_div_bounded_reward}
If the reward function is bounded in $[0,1]$,  then,  for any MDP $\Phi$,
\begin{align*}
    %\textnormal{MBR}_{\Phi} &\leq  \sum_{t=1}^T \bE_{ }  \Big[\sqrt{\frac{1}{2} \KL{\bP_{Y^\star_t, S^\star_t|\Theta = \theta}}{\bP_{\hat{Y}_t,\hat{S}_t|\hat{H}^t = \hat{h}^t}}} \Big]\\
    \textnormal{MBR}_{\Phi}&\leq  \sum_{t=1}^T \bE \bigg[\sqrt{\frac{1}{2} \KL{\bP_{Y^\star_t, S^\star_t|\Theta}}{\bP_{\hat{Y}_t,\hat{S}_t|\hat{H}^t}}} \bigg].
\end{align*}
\end{restatable}

Second,  we can note that a bounded $[0,1]$ function is $1$-Lispchitz under the discrete metric (or Hamming distortion) $\rho((y,a),(y',a')) \coloneqq \mathbbm{1}_{(y,a)=(y',a')}$ where $\mathbbm{1}$ is the indicator function.  Using this fact,  we can obtain Corollary~\ref{cor:wasserstein} from Proposition~\ref{prop:wasserstein}. 

\begin{restatable}{corollary}{wassersteincor}
\label{cor:wasserstein}
If the reward function is bounded in $[0,1]$,  then,  for any MDP $\Phi$,
\begin{equation*}
      \textnormal{MBR}_{\Phi}  \leq  \sum_{t=1}^T  \bE \big[\bW(\bP_{Y^\star_t, S^\star_t|\Theta}, \bP_{\hat{Y}_t,\hat{S}_t|\hat{H}^t}) \big].
\end{equation*}
%where the expectation is taken with respect to $\bP_{\Theta,\hat{H}^t}$.
\end{restatable}

\begin{remark}
\label{rem:wasserstein_is_tighter}
Corollary~\ref{cor:wasserstein} provides a tighter bound than Corollary~\ref{cor:kl_div_bounded_reward}. Indeed, if the geometry
is ignored (i.e., the discrete metric is considered), then for all $t=1,\ldots,T$,
\begin{align*}
    \bE \big[\bW(\bP_{Y^\star_t, S^\star_t|\Theta},& \bP_{\hat{Y}_t,\hat{S}_t|\hat{H}^t}) \big]\\
    &= \bE \big[\TV(\bP_{Y^\star_t, S^\star_t|\Theta}, \bP_{\hat{Y}_t,\hat{S}_t|\hat{H}^t}) \big]\\
    &\leq \bE \bigg[ \sqrt{\frac{1}{2}\KL{\bP_{Y^\star_t, S^\star_t|\Theta}}{\bP_{\hat{Y}_t,\hat{S}_t|\hat{H}^t}}} \bigg],
    %&\leq \sqrt{\frac{1}{2}\KL{\bP_{Y^\star_t, S^\star_t|\Theta}}{\bP_{\hat{Y}_t,\hat{S}_t|\hat{H}^t}}}
\end{align*}
where the equality follows from~\cite[Proof of Theorem~6.15]{villani2009optimal} and inequality follows from Pinsker’s~\cite[Theorem~6.5]{polyanskiy2014lecture} and Bretagnolle–Huber’s result~\cite[Proof of Lemma~2.1]{bretagnolle1978estimation}.
\end{remark}

\section{Upper bounds for static MDPs}
\label{sec:upper_bounds_MAB}

In this section, we leverage the bound from Section~\ref{sec:upper_bounds} to obtain bounds on the minimum Bayesian regret for static Markov decision processes. We focus here on the case where the reward function is bounded in $[0,1]$ and leave the sub-Gaussian and Lipschitz cases to the Appendix~\ref{sec:static_mdps_subg_lip}, since they are analogous to the previous section. We first present upper bounds on the MBR for the multi-armed bandit problem. We then produce upper bounds to the online optimization with partial feedback problem, and show how they can recover \emph{from below} the results from Russo and Van Roy~\cite{russo2016information}.

Similarly to Section~\ref{sec:upper_bounds}, we can apply Lemma~\ref{lemma:thompson_sampling} to upper bound the MBR for static MDPs. In the case of a static MDP, the right-hand side of that bound can be written as
\begin{align*}
    \sum_{t=1}^T \bE \bigg[ \bE \Big[ r\big(Y_t, \psi^\star_t (S, \Theta)\big) - r\big(Y_t, \psi^\star_t (S, \hat{\Theta}_t) \big)\Big| \Theta, \hat{\Theta}_t, \hat{H}^t \Big] \bigg].
    %\textnormal{MBR}_\Pi %&\leq R_\Phi(\kappa_\Theta) - r_\Phi(\kappa_\textnormal{H},\kappa_{\Theta|\textnormal{H}},\psi^\star)\\
    %&\leq \sum_{t=1}^T \bE \bigg[ \bE \Big[ r\big(Y^\star_t, \psi^\star_t (S^\star_t, \Theta)\big)| \Theta, \hat{H}^t \Big]\\
    %&- \bE \Big[ r\big(\hat{Y}_t, \psi^\star_t (\hat{S}_t, \hat{\Theta}_t) \big)\Big| \hat{\Theta}_t, \hat{H}^t \Big] \bigg]
    %\bE \Big[\sum_{t=1}^T r(Y_t,\psi_t^\star(S,\Theta))\Big]-\bE\Big[\sum_{t=1}^T r(Y_t,\psi_t^\star(S,\hat{\Theta}_t))\Big]
\end{align*}
This rewriting of the bound is obtained the same way as ~\eqref{eq:mdp_diff_expectations_thompson}: unrolling $R_\pi(\kappa_\Theta)$ and $r_\Pi(\kappa_\textnormal{H}, \kappa_{\textnormal{H}|\Theta}, \psi^\star)$, using the linearity of the expectation, the law of total expectation and the fact that the state $S$ does not depend on the time $t=1,\ldots,T$.

In the case where the outcomes $\lbrace Y_t \rbrace_{t=1,\ldots,T}$ do not depend on the state $S$, as in a MAB problem, it is possible to rewrite the actions taken by optimal policy $\psi^\star_t(S,\Theta)$ as $\gamma^\star(\Theta)$, where the function $\gamma^\star: \cO \to \cA$ is such that for all $S \in \cS$ and all $\Theta \in \cO$, it holds that $\psi^\star_t(S,\Theta) = \gamma^\star(\Theta)$.
%We will use $A^\star = \gamma^\star(\Theta)$ to denote the Bayesian optimal action and $\hat{A}_t = \gamma^\star(\hat{\Theta}_t)$ to be the action given the Thompson sampling strategy. 
In that case, it comes that the right-hand side term from Lemma 2 can be written as
\begin{equation}
    \sum_{t=1}^T \bE  \Big[ \bE  \big[ r(Y_t,A^\star) -  r(Y_t,\hat{A}_t)\big]|A^\star,\hat{A}_t,\hat{H}^t \Big].
    %\sum_{t=1}^T \bE  \Big[ \bE  \big[ r(Y_t,\gamma^\star(\Theta)) -  r(Y_t,\gamma^\star(\hat{\Theta})\big]|\Theta,\hat{\Theta}_t,\hat{H}^t \Big].
    \label{eq:smdp_diff_expectations_thompson}
\end{equation}

\begin{remark}
\label{rem:smdp_diff_dists}
Under this reformulation, the outcome in the first term inside the conditional expectation is distributed according to $\bP_{Y_t|A^\star,\hat{H}^t}$ and the second term is distributed according to $\bP_{Y_t|\hat{H}^t}$. This happens since $Y_t$ is independent of the sampled environment parameters $\hat{\Theta}_t$, and therefore independent of the sampled action $\hat{A}_t$ when the history $\hat{H}^t$ is known. 
\end{remark}

% \begin{remark}
% \label{rem:smdp_diff_dists}
% Under this reformulation, the outcome in the first term inside the conditional expectation is distributed according to $\bP_{Y_t|\Theta}$ since $Y_t$ are independent of the history $\hat{H}^t$ once the environment parameters are known.  
% The second term is distributed according to $\bP_{\hat{Y}_t|\hat{H}^t}$ since $\hat{Y}_t$ is independent of the sampled environment parameters $\hat{\Theta}_t$ when the history is known. 
% \end{remark}

% \textcolor{black}{
% Text about how to arrive to equation:
% write a similar remark to remark 2 explaining which two distributions are now the expectations of each term
% }

\subsection{Multi-armed bandit problem}

In this subsection, we propose minimum Bayesian regret bounds for multi-armed bandit problems $\Pi$. 

The tightest bound we obtain relates the $\textnormal{MBR}_\Pi$ to the  Wasserstein distance between the conditional probability of the outcome given the optimal action and the history collected following a Thompson sampling policy, and the conditional probability of the outcome given only the history.

\begin{restatable}{proposition}{MABwassersteinBounded}
\label{prop:MABwassersteinBounded}
If the reward function is bounded in $[0,1]$,  then for any static MDP $\Pi$,
\begin{equation*}
      \textnormal{MBR}_{\Pi}  \leq \sum_{t=1}^T  \bE \big[\bW(\bP_{Y_t|A^\star,\hat{H}^t}, \bP_{Y_t|\hat{H}^t}) \big].
\end{equation*}
%where the expectation is taken with respect to the joint distribution $\bP_{A^\star,\hat{H}^t}$.
\end{restatable}

% \begin{restatable}{proposition}{MABwassersteinBounded}
% \label{prop:MABwassersteinBounded}
% If the reward function is bounded in $[0,1]$,  then for any static MDP $\Pi$,
% \begin{equation*}
%       \textnormal{MBR}_{\Pi}  \leq \sum_{t=1}^T  \bE \big[\bW(\bP_{Y_t|\Theta}, \bP_{Y_t|\hat{H}^t}) \big]
% \end{equation*}
% where the expectation is taken with respect to the joint distribution $\bP_{\Theta,\hat{H}^t}$.
% \end{restatable}

\begin{proof}
\textcolor{black}{The proof follows from applying  Kantorovich–Rubinstein duality~\cite[Remark~6.5]{villani2009optimal} to~\eqref{eq:smdp_diff_expectations_thompson} using~\Cref{rem:smdp_diff_dists} in the same way as~\cite{rodriguez2021tighter,wang2019information}.}
\end{proof}

Using the same arguments as in Remark~\ref{rem:wasserstein_is_tighter}, together with Jensen's inequality, one can relax the bound from Proposition~\ref{prop:MABwassersteinBounded} and relate the $\textnormal{MBR}_\Pi$ to the conditional mutual information between the outcome $Y_t$ and the optimal action $A^\star$ given the history $\hat{H}^t$. This is formalized in the following corollary.

\begin{restatable}{corollary}{MABbounded}
\label{cor:MAB_bounded}
If the reward function is bounded in $[0,1]$,  then for any static MDP $\Pi$,
\begin{align*}
    %\textnormal{MBR}_{\Pi}    &\leq \sum_{t=1}^T \bE \bigg[ \sqrt{\frac{1}{2} \mi(Y_t;A^\star|\hat{H}^t)}\bigg].
    \textnormal{MBR}_{\Pi}    &\leq \sum_{t=1}^T \sqrt{\frac{1}{2} \mi(Y_t;A^\star|\hat{H}^t)}.
\end{align*}

\end{restatable}
This conditional mutual information can be interpreted as the remaining ``amount of surprise about the output $Y_t$" after observing the history $\hat{H}^t$ that is removed when the optimal action $A^\star$ is revealed. 

\subsection{Online optimization with partial feedback problem}

In the special case of online optimization with partial feedback, the right-hand-side term from Lemma \ref{lemma:thompson_sampling} in ~\eqref{eq:smdp_diff_expectations_thompson} can be formulated in a compact form using the preference function:
\begin{equation}
\label{eq:oopf_diff_expectations_thompson}
    \sum_{t=1}^T \bE  \Big[  \bE  \big[ r'(Y_{t,A^\star})\big]-\bE \big[ r'(Y_{t,\hat{A}_t})\big]|A^\star,\hat{A}_t,\hat{H}^t \Big].
\end{equation}

\begin{remark}
\label{rem:oopf_diff_dists}
In this last rewriting, the outcome in the first term inside the conditional expectation is distributed according to $\bP_{Y_{t,A^\star}|A^\star,\hat{H}^t}$ and the second term is distributed according to $\bP_{Y_{t,\hat{A}_t}|\hat{H}^t}$. This holds since $Y_t$ is independent of the sampled environment parameters $\hat{\Theta}_t$, and therefore independent of the sampled action $\hat{A}_t$ when the history $\hat{H}^t$ is known. 
\end{remark}

%For the online optimization with partial feedback problem, 

As the terms in~\eqref{eq:oopf_diff_expectations_thompson} are a difference of expectations of random objects which randomness comes from distributions on the same space, we can upper bound the minimum Bayesian regret using the Wasserstein distance
in terms of such distributions following the techniques from~\cite{rodriguez2021tighter,wang2019information}. 
%
%between the conditional probability of the optimal action given the ``per-action outcome" and the conditional probability of the optimal action given the history.
This is formalized in the following proposition.

\begin{restatable}{proposition}{RUSSOwassersteinBounded}
\label{prop:Russo_wasserstein_bounded}
If the reward function is bounded in $[0,1]$, then for any \emph{online optimization problem with partial feedback} $\Pi$,
\begin{equation*}
      \textnormal{MBR}_{\Pi}  \leq \sum_{t=1}^T  \bE \big[\bW(\bP_{Y_{t,A^\star}|A^\star,\hat{H}^t}, \bP_{Y_{t,A^\star}|\hat{H}^t})\big].
      %\textnormal{\textcolor{black}{ under construction}} 
\end{equation*}
%where the expectation is taken with respect to $\bP_{A^\star,\hat{H}^t}$.
\end{restatable}

\begin{proof}
\textcolor{black}{The proof follows from applying  Kantorovich–Rubinstein duality~\cite[Remark~6.5]{villani2009optimal} to~\eqref{eq:oopf_diff_expectations_thompson} using~\Cref{rem:oopf_diff_dists}.}
\end{proof}

This bound can also be relaxed following a similar procedure as Remark \ref{rem:wasserstein_is_tighter} to relate the $\textnormal{MBR}_\Pi$ to the relative entropy between the distribution of the ``per-action outcome" $Y_{t,A^\star}$ given the optimal action $A^\star$ and the history $\hat{H}^t$ and given the history only.  

\begin{restatable}{corollary}{RUSSObounded}
\label{cor:Russo_bounded}
If the reward function is bounded in $[0,1]$, then for any \emph{online optimization problem with partial feedback} $\Pi$,
\begin{align*}
    \textnormal{MBR}_{\Pi}    \leq  \sum_{t=1}^T \bE \bigg[\sqrt{\frac{1}{2} \KL{\bP_{Y_{t,A^\star}|A^\star,\hat{H}^t}}{\bP_{Y_{t,A^\star}|\hat{H}^t}}} \bigg].
\end{align*}
\end{restatable}

As the above stated %Proposition \ref{prop:Russo_wasserstein_bounded} and
Proposition \ref{prop:Russo_wasserstein_bounded} is derived using Lemma \ref{lemma:thompson_sampling}, its bound naturally holds for the regret of the Thompson sampling algorithm, namely:
\begin{align*}
%\label{eq:thompson_bound}
    R_\Pi(\kappa_\Theta) - &r_\Pi(\kappa_\textnormal{H},\kappa_{\Theta|\textnormal{H}},\psi^\star)\leq \nonumber \\
    &\sum_{t=1}^T  \bE \big[\bW(\bP_{Y_{t,A^\star}|A^\star,\hat{H}^t}, \bP_{Y_{t,A^\star}|\hat{H}^t})\big]
\end{align*}

We can further relax %the bound in~\eqref{eq:thompson_bound} 
this bound to recover results from Russo and Van Roy~\cite{russo2016information}. More specifically, we can recover the general bound combining~\cite[Propositions~1 and~3]{russo2016information}, and the specific bound combining~\cite[Propositions~1 and~4]{russo2016information}, for which it is assumed that the outcome $Y_t$ is perfectly revealed upon observing $Y_{t,a}$ for any $a\in\cA$. These claims are formalized in Corollary \ref{cor:propositions_from_Russo}. 

\begin{restatable}{corollary}{RUSSOboundedCorollary}
\label{cor:propositions_from_Russo}
If the reward function is bounded in $[0,1]$, then for any \emph{online optimization problem with partial feedback} $\Pi$, we have the following inequality on the bound from Proposition \ref{prop:Russo_wasserstein_bounded}: 
\begin{align*}
   \sum_{t=1}^T  \bE \big[\bW(\bP_{Y_{t,A^\star}|A^\star,\hat{H}^t}, \bP_{Y_{t,A^\star}|\hat{H}^t})\big] \leq\sqrt{\frac{1}{2}  |\cA| \ent(A^\star) T  }.
\end{align*}
Under the additional assumption that the outcome $Y_t$ is perfectly revealed upon observing $Y_{t,a}$ for any $a\in\cA$, one can obtain a tighter result: 
\begin{align*}
    \sum_{t=1}^T  \bE \big[\bW(\bP_{Y_{t,A^\star}|A^\star,\hat{H}^t}, \bP_{Y_{t,A^\star}|\hat{H}^t})\big] \leq\sqrt{\frac{1}{2} \ent(A^\star) T  }.
\end{align*}

\end{restatable}
\begin{proof}[Intuition of the proof]
For both results, the proof starts from using the same steps as Remark \ref{rem:wasserstein_is_tighter} to relax the bound from Proposition \ref{prop:Russo_wasserstein_bounded}. Then, both of the proofs rely on the application of Cauchy-Schwartz's and Jensen's inequalities to obtain a bound using the sum of the  conditional mutual information between the optimal action $A^\star$ and the “per-action outcome” $Y_{t,A_t}$ given the history $\hat{H}^t$. One can then show that the entropy of the optimal action $\ent(A^\star)$ upper bounds this sum of conditional mutual information $\mi(A^\star; Y_{t,A_t}|\hat{H}^t)$ to obtain the desired results. The assumption that the outcome $Y_t$ is perfectly revealed upon observing $Y_{t,a}$ for any $a\in\cA$ averts an extra use of the Cauchy-Schwartz inequality and thus allows to avoid an explicit dependence on the the number of actions through the multiplicative constant $\sqrt{|\cA|}$.
The full proof can be found in Appendix~\ref{sec:oopf_mdps_subg_lip}.
\end{proof}

\section{Conclusion}
\label{sec:conclusion}

In this paper, building on the results from~\cite{xu2020minimum}, we introduce a framework to study the Bayesian cumulative reward and the minimum Bayesian regret for reinforcement learning problems in the form of Markov decision process. The latter, is an algorithm-independent quantity and reflects the difficulty of the reinforcement learning problem.
We prove a data processing inequality for the Bayesian cumulative reward and present upper bounds on the minimum Bayesian regret using the Wasserstein distance and the relative entropy. We leverage these  results to the particular cases of the multi-armed bandit and the online optimization with partial feedback problems. For this last problem, our bound can be relaxed to recover from below the results presented in~\cite{russo2016information}.

\IEEEtriggeratref{12}
\bibliographystyle{IEEEtran}
\bibliography{refs}

\onecolumn
\appendices
\renewcommand*{\proofname}{Proof}

\section{Lemmas and extra remarks}
\label{sec:proofs_lemas}

\dataprocess*

\begin{proof}
The proof starts by writing explicitly the inequality to be proven, namely
\begin{equation*}
    R_\Phi(\kappa_U) = \sup_{ \lbrace \varphi_t \rbrace_{t=1}^T} \bE \bigg[ \sum_{t=1}^T r\big(Y_t, \varphi_t(S_t, U^t) \big) \bigg] \geq \sup_{ \lbrace \psi_t \rbrace_{t=1}^T} \bE \bigg[ \sum_{t=1}^T r\big(Y_t', \psi_t(S_t', V_t) \big) \bigg] = R_\Phi(\kappa_U, \kappa_{V|U}),
\end{equation*}
where $Y_t$ and $S_t$ are the outcomes and states obtained from the actions derived from $\varphi_t$, the kernels that describe the MDP $\Phi$, and $\kappa_U$; where $Y_t'$, $S_t'$, and $V_t$ are the outcomes, the states, and processed knowledge obtained from the actions derived from $\psi_t$ and the kernels that describe the MDP $\Phi$, and $\kappa_U$ and $\kappa_{V|U}$. 

Now, the proof follows by iteratively employing~\cite[Lemma~3.22]{kallenberg2005probabilistic} in a similar fashion to~\cite[Lemma~1]{xu2020minimum}. To start, consider the kernel $\kappa_{V|U,1}$ from $\cU$ to $\cV$, which are assumed to be Borel spaces. Then, there exists a measurable function $f_1: \cU \times [0,1] \to \cV$ such that if $\Xi \sim \textnormal{Uniform}[0,1]$ then $f(u,\Xi) \sim \kappa_{V|U,1}(\cdot,u)$ for all $u \in \cU$~\cite[Lemma~3.22]{kallenberg2005probabilistic}. Then,
\begin{align}
    R_\Phi(\kappa_U) &= \sup_{ \lbrace \varphi_t \rbrace_{t=1}^T} \bE \bigg[ r\big(Y_1', \varphi_1(S_1', U^1)) + \sum_{t=2}^T r\big(Y_t, \varphi_t(S_t, U^t) \big) \bigg] 
    \label{eq:prime_initial}
    \\
    &\geq \sup_{ \psi_1, \lbrace \varphi_t \rbrace_{t=2}^T} \sup_{\xi_1 \in [0,1]} \bE \bigg[ r\big(Y_1', \psi_1(S_1', f_1(U^1, \xi))) + \sum_{t=2}^T r\big(\bar{Y}_t, \varphi_t(\bar{S}_t, \bar{U}^t) \big) \bigg]
    \label{eq:func_change}
    \\ 
    &\geq \sup_{ \psi_1, \lbrace \varphi_t \rbrace_{t=2}^T} \bE \bigg[ r\big(Y_1', \psi_1(S_1', f_1(U^1, \Xi))) + \sum_{t=2}^T r\big(\tilde{Y}_t, \varphi_t(\tilde{S}_t, \tilde{U}^t) \big) \bigg] 
    \label{eq:func_change_exp}
    \\
    &= \sup_{ \psi_1, \lbrace \varphi_t \rbrace_{t=2}^T} \bE \bigg[ r\big(Y_1', \psi_1(S_1', V_1)) + r\big(Y_2', \varphi_2(S_2', {U'}^2)) + \sum_{t=3}^T r\big(\tilde{Y}_t, \varphi_t(\tilde{S}_t, \tilde{U}^t) \big) \bigg],
    \label{eq:prime_second}
\end{align}
where~\eqref{eq:prime_initial} follows since neither $S_1$ and $Y_1$ nor $S_1'$ and $Y_1'$ depend on the actions derived from $\varphi$ or $\psi$, and therefore $S_1, S_1' \sim \bP_{S|\Theta}$ and $Y_1,Y_1' \sim \bP_{Y_1|S_1,\Theta}$. Then,~\eqref{eq:func_change} follows since the supremum of functions $\varphi_1 : \cS \times \cU \to \cA$ is restricted to the functions $\sup_\psi \sup_\xi \psi(\cdot, f_1(\cdot, \xi))$, and where $\bar{Y}_t$, $\bar{S}_t$, and $\bar{U}$ denote the outcomes, states, and knowledge obtained from the actions derived from $\psi_1$ and $\varphi_t$ for $t \geq 2$, the kernels that describe the MDP $\Phi$, and $\kappa_U$. After that,~\eqref{eq:func_change_exp} holds since $\Xi$ is independent of all random objects and $\sup_x f(x) \geq \bE[f(X)]$. Here, $\tilde{Y}_t$, $\tilde{S}_t$, and $\tilde{U}^t$ denote the outcomes and states obtained from the actions derived from $\psi_1$ and $\varphi_t$ for $t \geq 2$, the kernels that describe the MDP $\Phi$, and $\kappa_U$. Finally, in~\eqref{eq:prime_second} it is used that $V_1 = f(U^1, \Xi)$ and therefore that $S_2' = \tilde{S}_2$ and $Y_2' = \tilde{Y}_2'$. Similarly, ${U'}^2$ is the knowledge obtained from $\psi_1$, the kernels that describe the MDP $\Phi$, and $\kappa_U$ and $\kappa_{V|U,1}$. 

Repeating the technique above focusing on $r(Y_2',\varphi_2(S_2',{U'}^2))$ and using~\cite[Lemma~3.22]{kallenberg2005probabilistic} with the kernel $\kappa_{V|U,2}$ from $\cU^2$ to $\cV$ one obtains that
\begin{align}
    R_\Psi(\kappa_U) \geq \sup_{ \lbrace \psi_t \rbrace_{t=1}^2, \lbrace \varphi_t \rbrace_{t=3}^T} \bE \bigg[ \sum_{t=1}^2 r\big(Y_t', \psi_1(S_t', V_t)) + r\big(Y_3', \varphi_3(S_3', {U'}^3)) + \sum_{t=4}^T r\big(\tilde{Y}_t, \varphi_t(\tilde{S}_t, \tilde{U}^t) \big) \bigg],
    \label{eq:prime_third}
\end{align}
where the notation is abused and $\tilde{Y}_t$ and $\tilde{S}_t$ denote the outcomes and states obtained from the actions derived from $\psi_1$, $\psi_2$, and $\varphi_t$ for $t \geq 3$, the kernels that describe the MDP $\Phi$, and $\kappa_U$ and $\kappa_{V|U,1}$ for $t < 3$. Also as before, ${U'}^3$ is the knowledge obtained from $\psi_t$ for $t < 3$, the kernels that describe the MDP $\Phi$, and $\kappa_U$ and $\kappa_{V|U,t}$ for $t < 3$. 

Finally, iterating this technique results in 
\begin{equation}
    R_\Phi(\kappa_U) \geq \sup_{ \lbrace \psi_t \rbrace_{t=1}^T} \bE \bigg[ \sum_{t=1}^T r\big(Y_t', \psi_t(S_t', V_t) \big) \bigg] = R_\Phi(\kappa_U, \kappa_{V|U})
    \label{eq:prime_final}
\end{equation}
and completes the proof.
\end{proof}

\begin{remark}
In the proof, $\psi_1$ may be different in both~\eqref{eq:prime_second} and~\eqref{eq:prime_third}, since the supremum may vary. However, $Y_2'$, $S_2'$, and ${U'}^2$ still represent the outcome, the state, and the knowledge obtained from the action derived from $\psi_1$, the kernels that describe the MDP, and $\kappa_U$ and $\kappa_{V|U,1}$. The same is true for all $Y_t'$, $S_t'$, and ${U'}^t$ along the proof, ensuring that the random objects in~\eqref{eq:prime_final} are distributed as in the definition of $R_\Phi(\kappa_U,\kappa_{V|U})$.
\end{remark}

\section{Upper bounds for static MDPs for sub-Gaussian and Lipschitz losses}
\label{sec:static_mdps_subg_lip}

\subsection{Multi-armed bandit problems}
\label{sec:mab_mdps_subg_lip}

\begin{restatable}{proposition}{MABsubgaussian}
\label{prop:MAB_sugaussian}
If for all $t = 1,\ldots,T$, the random reward $r(Y,A^\star)$ is $\sigma_t^2$-sub-Gaussian under $\bP_{Y_t|\hat{H}^t = \hat{h}^t}$ for all $\theta \in \cO$ and all $\hat{h}^t \in \cH^t$, then for any static MDP $\Pi$,
\begin{align*}
    \textnormal{MBR}_{\Pi}    &\leq \sum_{t=1}^T \sqrt{2 \sigma_t^2 \mi(Y_t;A^\star|\hat{H}^t)}.
\end{align*}

\end{restatable}

\begin{proof}

The proof starts from applying Donsker-Varadhan's inequality~\cite[Theorem~5.2.1]{gray2011entropy} to~\eqref{eq:smdp_diff_expectations_thompson} using~\Cref{rem:smdp_diff_dists} in the same way as~\cite{russo2016information,xu2017information}. The last inequality is obtained using Jensen's inequality and identifying the conditional mutual information between the outcome $Y_t$ and the optimal action $A^\star$ given the history $\hat{H}^t$. Namely, 
\begin{align*}
    \textnormal{MBR}_\Pi \leq \sum_{t=1}^T \bE  \Big[ \bE  \big[ r(Y_t,A^\star) -  r(Y_t,\hat{A}_t)\big]|A^\star,\hat{A}_t,\hat{H}^t \Big]
    \leq \sum_{t=1}^T \bE  \Big[ \sqrt{2\sigma^2_t\KL{\bP_{Y_t|A^\star,\hat{H}^t}}{\bP_{Y_t|\hat{H}^t})}} \Big]
    \leq \sum_{t=1}^T \sqrt{2 \sigma_t^2 \mi(Y_t;A^\star|\hat{H}^t)}.
\end{align*}
\end{proof}

\begin{restatable}{proposition}{MABwasserstein}
\label{prop:MABwasserstein}
Suppose that $\cY$ is a metric space with metric $\rho$.  If the reward function $r: \cY \times \cA \to \bR$ is $L$-Lipschitz in $\cY$ under the metric $\rho$, then
\begin{equation*}
      \textnormal{MBR}_{\Pi}  \leq L \sum_{t=1}^T \bE\big[\bW(\bP_{Y_{t}|A^\star}, \bP_{Y_{t}|\hat{H}^t})\big].
\end{equation*}
\end{restatable}

\begin{proof}
The proof follows from applying  Kantorovich–Rubinstein duality~\cite[Remark~6.5]{villani2009optimal} to~\eqref{eq:smdp_diff_expectations_thompson} using~\Cref{rem:smdp_diff_dists} analogously  to~\cite{rodriguez2021tighter,wang2019information}.
\end{proof}

\subsection{Online optimization with partial feedback problems}
\label{sec:oopf_mdps_subg_lip}

\begin{restatable}{proposition}{RUSSOsubgaussian}
\label{prop:Russo_sugaussian}
If for all $t = 1,\ldots,T$, the random reward $r(Y,a^\star)$ is $\sigma_t^2$-sub-Gaussian under $\bP_{Y|\hat{H}^t = \hat{h}^t}$ for all $a^\star \in \cA$ and all $\hat{h}^t \in \cH^t$, then for any \emph{online optimization problem with partial feedback} $\Pi$,
\begin{align*}
    \textnormal{MBR}_{\Pi}    &\leq \sum_{t=1}^T \bE \Big[\sqrt{2 \sigma_t^2  \KL{\bP_{Y_{t,A^\star}|\Theta}}{\bP_{Y_{t,A^\star}|\hat{H}^t}}}\Big] %\textcolor{black}{Incomplete}.
\end{align*}

\end{restatable}

\begin{proof}

The proof follows from applying Donsker-Varadhan's inequality~\cite[Theorem~5.2.1]{gray2011entropy} to~\eqref{eq:oopf_diff_expectations_thompson} using~\Cref{rem:oopf_diff_dists} in a similar fashion to~\cite{russo2016information,xu2017information}.
\end{proof}

\begin{restatable}{proposition}{RUSSOwasserstein}
\label{prop:Russo_wasserstein}
Suppose that $\cY$ is a metric space with metric $\rho$.  If the reward function $r: \cY \times \cA \to \bR$ is $L$-Lipschitz in $\cY$ under the metric $\rho$, then for any \emph{online optimization problem with partial feedback} $\Pi$
\begin{equation*}
      \textnormal{MBR}_{\Pi}  \leq L \sum_{t=1}^T  \bE \big[\bW(\bP_{Y_{t,A^\star}|A^\star}, \bP_{Y_{t,A^\star}|\hat{H}^t})\big].
\end{equation*}
\end{restatable}

\begin{proof}
The proof follows from applying  Kantorovich–Rubinstein duality~\cite[Remark~6.5]{villani2009optimal} to~\eqref{eq:oopf_diff_expectations_thompson} using~\Cref{rem:oopf_diff_dists} analogously  to~\cite{rodriguez2021tighter,wang2019information}.
\end{proof}

\begin{remark}
\label{rem:entropy_dominates}
One can show that the entropy of the optimal action $\ent(A^\star)$ upper bounds the sum of conditional mutual information between the optimal action $A^\star$ and the “per-action outcome” $Y_{t,A_t}$ given the history $\hat{H}^t$. This result is obtained in the same way as in \cite{russo2016information} through the following chain of inequalities,
\begin{align*}
    \sum_{t=1}^T \mi(A^\star; Y_{t,A_t}|\hat{H}^t) \stack{a}{\leq }\sum_{t=1}^T \mi(A^\star; (Y_{t,A_t},A_t)|\hat{H}^t)
    \stack{b}{=} \mi(A^\star; \{Y_{t,A_t},A_t\}_{t=1}^T)  
    \stack{c}{\leq} \ent(A^\star) 
\end{align*}
where~(a) follows from \cite[Theorem 2.3.5]{polyanskiy2014lecture}, equality (b) is given by the chain rule and~(c) is obtained from \cite[Theorem 2.4.4]{polyanskiy2014lecture}.
%The above result will come handy in the proofs of the two inequalities in Corollary \ref{cor:propositions_from_Russo}.
\end{remark}

\RUSSOboundedCorollary*
\begin{proof}

Under the assumption that the outcome $Y_t$ is perfectly revealed upon observing $Y_{t,a}$ for any $a\in\cA$, one can show the following chain of inequalities: 
\begin{align}
    \sum_{t=1}^T  \bE \big[\bW(\bP_{Y_{t,A^\star}|A^\star,\hat{H}^t}, \bP_{Y_{t,A^\star}|\hat{H}^t})\big] %&\leq \sum_{t=1}^T \bE \bigg[\sqrt{\frac{1}{2} \KL{\bP_{Y_{t,A^\star}|A^\star,\hat{H}^t}}{\bP_{Y_{t,A^\star}|\hat{H}^t}}} \bigg]\label{eq:cor_rem_3} \\
    &\leq \sum_{t=1}^T \bE \bigg[ \sqrt{\frac{1}{2} \KL{\bP_{Y_{t,A_t}|A^\star,\hat{H}^t}}{\bP_{Y_{t,A_t}|\hat{H}^t}}} \bigg]\label{eq:cor_rew_1} \\
    &\leq \sum_{t=1}^T  \sqrt{\frac{1}{2} \mi(A^\star; Y_{t,A_t}|\hat{H}^t)}\label{eq:cor_jens_1} \\
    &\leq \sqrt{\frac{1}{2} T \sum_{t=1}^T \mi(A^\star; Y_{t,A_t}|\hat{H}^t)} \label{eq:cor_cauchy}\\
    &\leq \sqrt{\frac{1}{2} T \ent(A^\star)} \nonumber
\end{align}
where~\eqref{eq:cor_rew_1} is obtained using the same arguments as in Remark \ref{rem:wasserstein_is_tighter} %, the rewriting in~\eqref{eq:cor_rew_1} follows as from assumption
and, as $Y_t$ is perfectly revealed from observing $Y_{t,a}$ for any $a\in \cA$, we have that
$\KL{\bP_{Y_{t,a}|A^\star,\hat{H}^t}}{\bP_{Y_{t,a}|\hat{H}^t}} = \KL{\bP_{Y_{t}|A^\star,\hat{H}^t}}{\bP_{Y_{t}|\hat{H}^t}}$ %does not depend on the sampled action $a\in \cA$
. Then Jensen's inequality leads to~\eqref{eq:cor_jens_1} and Cauchy-Schwartz inequality to~\eqref{eq:cor_cauchy}. Finally, applying Remark \ref{rem:entropy_dominates} yields the desired result. 

When no information structure is assumed among the outcome $Y_t \equiv \{Y_{t,a}\}$ for all $a\in \cA$, inspired by the arguments used to prove~\cite[Proposition 3]{russo2016information}, one can show a looser inequality, through a chain of inequalities : 

\begin{align}
    \sum_{t=1}^T  \bE \big[\bW(\bP_{Y_{t,A^\star}|A^\star,\hat{H}^t}, \bP_{Y_{t,A^\star}|\hat{H}^t})\big] %&\leq \sum_{t=1}^T \bE \bigg[\sqrt{\frac{1}{2} \KL{\bP_{Y_{t,A^\star}|A^\star,\hat{H}^t}}{\bP_{Y_{t,A^\star}|\hat{H}^t}}} \bigg] \label{eq:cor_2_rem_3} \\
    &\leq \sum_{t=1}^T \bE \Bigg[ \sum_{a\in \cA} \bP(A^\star = a|\hat{H}^t) \sqrt{\frac{1}{2} \KL{\bP_{Y_{t,a}|A^\star=a,\hat{H}^t}}{\bP_{Y_{t,a}|\hat{H}^t}}} \Bigg] \label{eq:cor_2_rem_3} \\
    &\leq \sum_{t=1}^T \bE \Bigg[ \sqrt{\frac{1}{2} |\cA| \sum_{a\in \cA} \bP(A^\star = a|\hat{H}^t)^2  \KL{\bP_{Y_{t,a}|A^\star=a,\hat{H}^t}}{\bP_{Y_{t,a}|\hat{H}^t}}} \Bigg]\label{eq:cor_2_cauchy} \\
    &\leq \sum_{t=1}^T \bE \Bigg[ \sqrt{\frac{1}{2} |\cA| \sum_{a,b\in \cA} \bP(A^\star = a|\hat{H}^t) \bP(A^\star = b|\hat{H}^t) \KL{\bP_{Y_{t,a}|A^\star=b,\hat{H}^t}}{\bP_{Y_{t,a}|\hat{H}^t}}} \Bigg] \label{eq:cor_2_extra_terms}\\
    %&= \sum_{t=1}^T  \bE \bigg[ \sqrt{\frac{1}{2} |\cA| \KL{\bP_{Y_{t,A_t}|A^\star,\hat{H}^t}}{\bP_{Y_{t,A_t}|\hat{H}^t}}} \bigg] \label{eq:cor_2_kl}\\
    &\leq\sum_{t=1}^T \sqrt{\frac{1}{2} |\cA| \mi(A^\star; Y_{t,A_t}|\hat{H}^t)} \label{eq:cor_2_mi}\\
    &\leq \sqrt{\frac{1}{2} |\cA| T \sum_{t=1}^T \mi(A^\star; Y_{t,A_t}|\hat{H}^t)} \label{eq:cor_2_cauchy_2}\\
    &\leq \sqrt{\frac{1}{2} |\cA| T \ent(A^\star)} \nonumber %\label{eq:cor_2_rem_7}
\end{align}
where~\eqref{eq:cor_2_rem_3} follows from Remark \ref{rem:wasserstein_is_tighter}, ~\eqref{eq:cor_2_cauchy} is obtained using Cauchy-Schwartz inequality, and~\eqref{eq:cor_2_extra_terms} by adding the non-negative extra terms $ \frac{1}{2} |\cA| \sum_{a\in \cA}  \bP(A^\star = a|\hat{H}^t) \sum_{b\in \cA \setminus a} \bP(A^\star = b|\hat{H}^t) \KL{\bP_{Y_{t,a}|A^\star=b,\hat{H}^t}}{\bP_{Y_{t,a}|\hat{H}^t}}$ in the square root in~\eqref{eq:cor_2_cauchy}. Then, \eqref{eq:cor_2_mi} follows from using Jensen's inequality and identifying the conditional mutual information between the optimal action $A^\star$ and the “per-action outcome” $Y_{t,A_t}$ given the history $\hat{H}^t$. Lastly, Cauchy-Schwartz inequality leads to~\eqref{eq:cor_2_cauchy_2} and Remark \ref{rem:entropy_dominates} gives the claimed result. 

\end{proof}

\end{document}

%% file: packages.tex
\usepackage{amsmath, amsthm, amssymb, amsfonts} 
\usepackage{nicefrac}
\usepackage{mathtools}
\usepackage{bm, bbm}
\usepackage{thm-restate}
\usepackage[scr=boondoxo,scrscaled=1.05]{mathalfa}
\usepackage{upgreek}

\usepackage{dsfont}
\usepackage{cite}
\usepackage{enumitem}

\usepackage{xcolor}
\usepackage{comment}

\usepackage{blindtext}

\usepackage{tikz}
\usepackage{graphicx}
\usepackage{tikzscale}

\usepackage[nameinlink]{cleveref}

%% file: macros.tex
\definecolor{orange}{rgb}{1,0.4,0.0}

\DeclarePairedDelimiterXPP{\KL}[2]{D_\textnormal{KL}}{(}{)}{}{%
#1\:\delimsize\|\:#2% 
}
\DeclarePairedDelimiterXPP{\RD}[2]{D_{$\alpha$}}{(}{)}{}{%
#1\:\delimsize\|\:#2%
}

\DeclarePairedDelimiterXPP\Prob[1]{\mathbb{P}}{\lbrace}{\rbrace}{}{

#1}
\DeclarePairedDelimiterXPP{\lnorm}[2]{}{\lVert}{\rVert}{_{#2}}{#1}

\newcommand{\TV}{\ensuremath{\textsc{\texttt{TV}}}}

\newcommand{\bD}{\ensuremath{\mathbb{D}}}
\newcommand{\bE}{\ensuremath{\mathbb{E}}}

\newcommand{\bP}{\ensuremath{\mathbb{P}}}
\newcommand{\bQ}{\ensuremath{\mathbb{Q}}}
\newcommand{\bR}{\ensuremath{\mathbb{R}}}

\newcommand{\bW}{\ensuremath{\mathbb{W}}}

\newcommand{\cA}{\ensuremath{\mathcal{A}}}
\newcommand{\cB}{\ensuremath{\mathcal{B}}}

\newcommand{\cH}{\ensuremath{\mathcal{H}}}

\newcommand{\cO}{\ensuremath{\mathcal{O}}}

\newcommand{\cS}{\ensuremath{\mathcal{S}}}

\newcommand{\cU}{\ensuremath{\mathcal{U}}}
\newcommand{\cV}{\ensuremath{\mathcal{V}}}

\newcommand{\cX}{\ensuremath{\mathcal{X}}}
\newcommand{\cY}{\ensuremath{\mathcal{Y}}}
\newcommand{\cZ}{\ensuremath{\mathcal{Z}}}

\newcommand{\ccS}{\ensuremath{\mathscr{S}}}

\newcommand{\ccX}{\ensuremath{\mathscr{X}}}
\newcommand{\ccY}{\ensuremath{\mathscr{Y}}}
\newcommand{\ccZ}{\ensuremath{\mathscr{Z}}}

\newcommand{\mi}{\textup{I}}
\newcommand{\ent}{\textup{H}}

\newcommand{\stack}[2]{\stackrel{\mathclap{(#1)}}{#2}}

\newtheoremstyle{mytheoremstyle} % name
    {\topsep}                    % Space above
    {\topsep}                    % Space below
    {\itshape}                   % Body font
    {}                           % Indent amount
    {\bf}                        % Theorem head font
    {.}                          % Punctuation after theorem head
    {.5em}                       % Space after theorem head
    {}  % Theorem head spec (can be left empty, meaning ‘normal’)

\theoremstyle{mytheoremstyle}
\newtheorem{lemma}{Lemma}

\newtheorem{definition}{Definition}

\newtheorem{remark}{Remark}
\newtheorem{assumption}{Assumption}